\documentclass[sigconf, hyphens]{acmart}

\AtBeginDocument{%
  \providecommand\BibTeX{{%
    \normalfont B\kern-0.5em{\scshape i\kern-0.25em b}\kern-0.8em\TeX}}}

\copyrightyear{2020}
\acmYear{2020}
\setcopyright{acmlicensed}\acmConference[CIKM '20]{Proceedings of the 29th ACM International Conference on Information and Knowledge Management}{October 19--23, 2020}{Virtual Event, Ireland}
\acmBooktitle{Proceedings of the 29th ACM International Conference on Information and Knowledge Management (CIKM '20), October 19--23, 2020, Virtual Event, Ireland}
\acmPrice{15.00}
\acmDOI{10.1145/3340531.3412051}
\acmISBN{978-1-4503-6859-9/20/10}

\usepackage{appendix}
\usepackage{graphicx} 

\usepackage[utf8]{inputenc} 
\usepackage[T1]{fontenc}    
\usepackage{booktabs}       
\usepackage{amsfonts}       
\usepackage{nicefrac}       
\usepackage{microtype}      
\usepackage{xspace}

\usepackage{subcaption}
\usepackage{hyphenat}
\usepackage{amsmath}
\usepackage{amsthm}
\theoremstyle{plain}
\usepackage{enumitem}

\newtheorem{thm}{Theorem}
\newtheorem{notation}[thm]{Notation}

\newcommand{\defeq}{\mathrel{\vcentcolon=}}

\newcommand{\Cliqz}{Cliqz\xspace}
\newcommand{\HumanWeb}{Human Web\xspace}

\newcommand{\hide}[1]{}

\newcommand{\xhdr}[1]{\vspace{1.7mm}\noindent{{\bf #1.}}}

\newcommand{\ie}{{i.e.}\xspace}

\newcommand{\cf}{{cf.}\xspace}

\newcommand{\Secref}[1]{Sec.~\ref{#1}}

\newcommand{\Eqnref}[1]{Eq.~\ref{#1}}

\newcommand{\dblsecref}[2]{Sec.~\ref{#1} and \ref{#2}}

\newcommand{\Figref}[1]{Fig.~\ref{#1}}

\newcommand{\Appref}[1]{Appendix~\ref{#1}}

\usepackage{mathtools}
\DeclarePairedDelimiter\abs{\lvert}{\rvert}
\DeclarePairedDelimiter\norm{\lVert}{\rVert}
\DeclarePairedDelimiter{\floor}{\lfloor}{\rfloor}

\usepackage{amsmath}
\DeclareMathOperator{\sgn}{sgn}

\usepackage[normalem]{ulem} 

\newcommand{\R}{\mathbb{R}}
\newcommand{\N}{\mathbb{N}}

\author{Valentin Hartmann}
\orcid{0000-0001-9856-0089}
\affiliation{%
  \institution{EPFL}
}
\email{valentin.hartmann@epfl.ch}

\author{Konark Modi, Josep M. Pujol}
\affiliation{%
  \institution{Cliqz}
}
\email{{konark, josep}@cliqz.com}

\author{Robert West}
\orcid{0000-0002-3984-1232}
\affiliation{%
  \institution{EPFL}
}
\email{robert.west@epfl.ch}

\begin{document}

\fancyhead{}

%


\title{Privacy-Preserving Classification with Secret Vector Machines}


\begin{abstract}
Today, large amounts of valuable data are distributed among millions of user-held devices, such as personal computers, phones, or Internet-of-things devices. Many companies collect such data with the goal of using it for training machine learning models allowing them to improve their services. User-held data is, however, often sensitive, and collecting it is problematic in terms of privacy.
We address this issue by proposing a novel way of training a supervised classifier in a distributed setting akin to the recently proposed federated learning paradigm, but under the stricter privacy requirement that the server that trains the model is assumed to be untrusted and potentially malicious.
We thus preserve user privacy by design, rather than by trust.
In particular, our framework, called secret vector machine (SecVM), provides an algorithm for training linear support vector machines (SVM) in a setting in which data-holding clients communicate with an untrusted server by exchanging messages designed to not reveal any personally identifiable information.
We evaluate our model in two ways. First, in an offline evaluation, we train SecVM to predict user gender from tweets, showing that we can preserve user privacy without sacrificing classification performance. Second, we implement SecVM's distributed framework for the Cliqz web browser and deploy it for predicting user gender in a large-scale online evaluation with thousands of clients, outperforming baselines by a large margin and thus showcasing that SecVM is suitable for production environments.

\end{abstract}

\maketitle

\section{Introduction}
\label{sec:intro}

With the growing number of smartphones, intelligent cars and smart home devices, the amount of highly valuable data that is spread among many devices increases at a rapid pace. Those devices are typically in possession of end users and so is the data produced by and stored on them. Of course companies are interested in making use of this data, e.g., to detect usage patterns, make better\hyp informed business decisions, and ultimately improve their products.
As an example, consider the case of a web browser vendor wanting to infer demographics from users' browsing histories in order to automatically change the default behavior for hiding adult-only content from users inferred to be minors,
or to show relevant website suggestions to users based on their inferred age groups.


The classical way of building the necessary prediction model would be to collect the users' data on a central server and then run a machine learning algorithm on it. But this comes with severe disadvantages. First, the user has to put the necessary trust in the data\hyp collecting entity. Even in case of trust, the relationship between the two parties is very imbalanced; new regulations such as the European Union's General Data Protection Regulation (GDPR) \cite{eu:gdpr} and e-privacy frameworks try to rebalance the relationship to be more fair. But still, even in the case of perfect regulation, the collection of user data incurs a privacy risk. There are many ways in which privacy could be compromised: hacks leading to a data breach \cite{data:breaches}, disgruntled or unethical employees that use the data for their own benefit \cite{leak:insider}, companies going bankrupt and selling the data as assets \cite{leak:bankrupt}, and of course government-issued subpoenas and backdoors \cite{legal:theguardian, legal:verizon}. All in all, it is safe to assume that gathering users' data puts their privacy at risk, regardless of the comprehensiveness of the data management policies in place.


It is thus desirable to be able to build prediction models without learning any personally identifying information about the individual users whose data is used in the training process.
For instance, the party who is training the model should not be able to infer labels or feature values of individual users.
This requirement immediately precludes us from using the standard machine learning setup, where feature vectors for all users are stored in a feature matrix, labels in a label vector, and an optimization algorithm is then used to find the model parameters that minimize a given loss function.

The issue with the vanilla machine learning setup is that the party training the model sees all data---features and labels---at the same time, which typically makes it easy to infer user identities, even if the data is pseudo\hyp anonymized, i.e., if actual user ids have been replaced with random identifiers.
One way to achieve this is by tying together multiple features associated with the same user, as was the case with the now\hyp infamous AOL data breach, where users thought to be anonymized were identified by analyzing all their search queries together \cite{aol:privacy}.
Moreover, user privacy can also be compromised by correlating the pseudo\hyp anonymized data with external datasets, which was done with the Netflix movie rating dataset \cite{narayanan2008robust}.

A recently proposed way forward is given by the paradigm of federated learning \cite{mcmahan2016federated}.
Here, model fitting does not happen locally on the machines of a single party;
rather, it works on distributed data without the need for central aggregation.
In essence, federated learning models perform gradient descent in a distributed way, where, instead of sharing their raw data with the server that is building the model, clients only share the gradient updates necessary to improve the loss function locally for their personal data.

While this is a promising approach, it was not originally designed with the goal of preserving privacy.
Later extensions have addressed this issue, by injecting random noise into the data on the client-side before sending it to the server
\cite{abadi2016deep} or by using cryptography \cite{bonawitz2017practical}.

\xhdr{Present work: Secret vector machines (SecVM)}
This work proposes a novel and different approach to training a supervised machine learning classifier in a privacy\hyp preserving manner.
Crucially, in our setup, the server is assumed to be \textit{untrusted,} i.e., potentially malicious.
Our key observation is that support vector machines (SVM), a popular machine learning model, are particularly well\hyp suited for privacy\hyp preserving classification, due to the hinge loss function used in its objectives:
when features and labels are binary (as is often the case), the SVM gradient can only take on the three discrete values $-1$, 0, and 1, which means particularly little information about the client is revealed.

Starting from this observation, we identify additional ingredients necessary to maintain user privacy and design a distributed protocol for training SVMs.
We term our algorithm \textit{secret vector machine (SecVM)}.
As in the federated learning setup \cite{mcmahan2016federated}, the SecVM model is trained on the server side in collaboration with the data\hyp owning clients.
What sets SecVM apart from prior proposals for privacy\hyp preserving federated learning \cite{geyer2017differentially,mcmahan2017learning} is that it assumes an untrusted server and works without adding random noise.
Instead, it leverages the above observation regarding the SVM loss and
makes the recovery of the original data from the updates shared by clients impossible by means of feature hashing, splitting updates into small chunks, and sending them to the server according to a particular communication protocol. Other popular classification algorithms such as logistic regression or neural networks do not share the property of SVM of having an integer\hyp valued gradient, and are therefore not suited for preserving privacy in our setting.


\xhdr{Contributions}
Our main contributions are as follows.
\begin{itemize}
    \item We propose secret vector machines (SecVM), a novel method for \textbf{training linear SVM} classifiers with integer features
    in a \textbf{privacy\hyp preserving} manner
    (\dblsecref{sec:problem statement}{sec:solution}).
    \item In an \textbf{offline evaluation}, we apply SecVM to a large dataset of tweets in order to infer users' gender from the words contained in their tweets, showing that we can maintain user privacy without lowering classification performance, compared to a vanilla SVM model (\Secref{sec:twitter}).
    \item We implement SecVM's client--server setup for
the \Cliqz{} web browser \cite{cliqz}
    and successfully deploy it in a large-scale \textbf{online evaluation} with thousands of participating clients,
    outperforming baselines on a gender prediction task by a large margin and
    thus demonstrating the feasibility of our model in production settings (\Secref{sec:cliqz}).
\end{itemize}

By exploiting specific properties of support vector machines, our method overcomes the shortcomings of other classification models in a privacy\hyp preserving context (cf. discussion in \Secref{sec:discussion} for details). Moreover, due to their good performance on various classification tasks, SVMs are of high practical relevance for both researchers as well as industry practitioners, with recent applications in material sciences \cite{1757-899X-436-1-012020}, geology \cite{CHEN2017314}, remote sensing \cite{Liu2017} and medicine \cite{10.1371/journal.pone.0161501, 7755785, PMID:27919211} --- the latter being a particularly privacy\hyp sensitive area.
Also, although our exposition considers binary labels, it can readily be extended to the multi-class case, via schemes such as one\hyp vs.\hyp one or one\hyp vs.\hyp all \cite{hsu2002comparison}.

\section{Related work}
\label{sec:relwork}
There are two main approaches to the problem of extracting information while at the same time protecting privacy. In the first, an altered version of the data is released. In the second,  all data stays on the data owners' devices, but they actively participate in the information extraction procedure.

\xhdr{Releasing the data} Information about individuals can be hidden by perturbing the data randomly, as applied in learning decision trees \cite{agrawal2000privacy} and other settings \cite{dwork2006calibrating,Evfimievski:2002:RPP:772862.772869}.
The notion of \textit{\(k\)-anonymity} \cite{samarati1998protecting} requires each record to be indistinguishable from at least \(k-1\) other records, which can be achieved by suppressing or generalizing certain attributes. Its shortcomings in terms of privacy have been addressed by \textit{$l$-diversity} \cite{machanavajjhala2006diversity}, which has further been refined to \textit{$t$-closeness} \cite{li2007t}.

\xhdr{Keeping the data}
A recently proposed solution for training machine learning models on distributed data is called \textit{federated learning} (FL) \cite{mcmahan2016federated}: a server distributes the current version of the model to the data\hyp owning clients, which then return only updates to this model, rather than their raw data. While FL's original focus was not on privacy, algorithms for extending it in this direction have been proposed \cite{geyer2017differentially,mcmahan2017learning}. These extensions build on techniques due to Abadi et al.\ \cite{abadi2016deep}, by randomly sampling a subset of users that should participate in a training round, and adding random noise to their updates.
In this setting, a client's private data is not to be protected against a malicious server (the server is assumed to be trusted), but rather against other clients participating in the training.
FL extensions that come closer to our setting are based on local differential privacy (LDP) \cite{liu2020fedsel,truex2019hybrid}. Here, noise is added to the model updates sent by the users to prevent them from leaking sensitive information. The mathematical guarantee provided by these methods is based on the widely accepted differential privacy, but it comes at the cost of significantly reducing the classification performance of the models. Even when reducing the required amount of noise by sending only one dimension of each user's model update per training iteration \cite{liu2020fedsel}, the reported increase in misclassification rate as compared to non\hyp private training can be larger than 0.2 for an SVM. The same holds for combining LDP with threshold homomorphic encryption \cite{truex2019hybrid}, where the F1\hyp score of the trained SVM drops by more than 0.1. As we show in our experiments (\Secref{sec:twitter}), the performance drop when using SecVM is negligible. In addition, the aforementioned method based on threshold homomorphic encryption \cite{truex2019hybrid} as well as other solutions based on cryptography \cite{bonawitz2017practical} require a trusted public key infrastructure, whereas SecVM can be used with any readily available anonymization network such as TOR (see \Secref{sec:solution}).

\section{Problem statement}
\label{sec:problem statement}
We consider the problem of training a linear SVM on distributed data, i.e., the data is spread among multiple Internet\hyp connected client devices owned by individual parties. Each client $i$'s dataset consists of a feature vector $x_i$ and a label $y_i$.
For reasons that will become clear later, we assume that features take on integer values, which they do, e.g., if they represent counts.
In addition to the clients, there is a separate party, the server owner, who would like to use the data for training a classification model. While this intention is supported by the clients, they do not want the server to learn anything about individual clients.

Why do we say ``anything'' above, when in fact we only want to protect \textit{sensitive} client information?
First off, the server should certainly never know the complete feature vector together with the label.
However, the feature vector alone, or a part of it, could already contain sensitive information:
if only a subset of the features suffices to identify a user, then the other features contain additional (potentially sensitive) information.
Likewise, if a single feature is unique to a certain user, the feature--label combination gives away their label.
And finally, a single feature alone can already contain sensitive information, e.g., if features are strings users type into a text box.
We therefore formalize our privacy requirement as follows:
\begin{quote}
\textbf{Privacy requirement.}
\textit{The server must not be able to infer the label or any feature value for any individual client.}
\end{quote}

\xhdr{Server capabilities}
In order to infer user information, we assume a malicious server can
(1)~observe incoming traffic,
(2)~send arbitrary data to clients, and
(3)~control an arbitrarily large fraction of the clients by introducing forged clients.

\section{Proposed solution}
\label{sec:solution}

The loss function of a binary classification model often takes the form
\begin{equation}
    \label{eq:general_loss}
    J(w) = \frac{1}{N}\sum_{i=1}^N L(w, x_i, y_i) + \lambda R(w),
\end{equation}
where \(N\) is the number of training samples (in our case, users), \(x_i\in\mathbb{R}^d\) user $i$'s (in our case, integer) feature vector, \(y_i \in \{-1,1\}\) user $i$'s label, \(w\in\R^d\) the parameter vector of the model, \(L\) the loss for an individual sample, \(R\) a regularization function independent of the data, and \(\lambda > 0\) the regularization parameter.
When using a subgradient method to train the model, the update for dimension \(j\) becomes
\begin{equation}
    \label{eq:general_update}
    w_j \leftarrow w_j - \eta \left(\frac{1}{N} \sum_{i=1}^N \frac{\partial L(w, x_i, y_i)}{\partial w_j} + \lambda \frac{\partial R(w)}{\partial w_j}\right),
\end{equation}
where \(\eta > 0\) is the learning\hyp rate parameter. In the case of linear SVMs, we have \(L(w, x_i, y_i) = \max\{0, 1-y_i w^T x_i\}\) and
\begin{equation}
    \frac{\partial L(w, x_i, y_i)}{\partial w_j} = \delta(1-y_i w^T x_i) \, y_i x_{ij},
\end{equation}
where $\delta(x)=1$ if $x>0$, and $\delta(x)=0$ otherwise.

The key observation that federated learning \cite{mcmahan2016federated} and also our method rely on is that the update of \Eqnref{eq:general_update} is the sum over values that each only depend on the data of a single user $i$.
To train the model, we follow a process akin to federated learning:
\begin{enumerate}
    \item The server sends out the current model $w$ to the clients.
    \item Each client $i$ computes its local update \(\nabla_w L(w, x_i, y_i)\) and sends it back to the server.
    \item The server sums up the individual updates, computes \(\nabla_w R(w)\), and makes an update to the model.
\end{enumerate}
To meet our privacy requirements, we adopt a slightly more nuanced protocol, described next.
To begin with, we will work under the following \textbf{temporary assumptions,} which will later be dropped to accommodate real-world settings:

\begin{enumerate}
    \item[A1.] There is a trusted third party that first collects the update data (as specified below) from the clients in every training iteration and then sends it as a single database to the server.
    \item[A2.] The server honestly follows the above three-step training procedure, which essentially means that the server sends the same correctly updated weight vector to all clients.
\end{enumerate}

We also assume that the client code is available for inspection by the users to make sure the clients follow the proposed protocol. The server code, however, need not be public, since SecVM protects against a malicious server.

\xhdr{Hiding feature values from the server: hashing}
We do not want the server to learn any feature values. For this purpose, we make features less specific by grouping them via hashing. Assume we have \(d\) different features, i.e., \(w\in\R^d\). We then choose a number \(1 < k < d\) and reduce $w$'s dimension to \(\tilde{d}\defeq\floor{d/k}\) by applying a hash function to each feature index and taking the result modulo \(\tilde{d}\) as its new index (its bin). If several features of a user hash into one bin, their values are added. In expectation, this results in at least \(k\) features, indistinguishable from each other, per bin.

Usually, hashing is used to reduce the problem dimensionality, and thereby resource consumption, and collisions are a necessary but unwanted side effect, as they may decrease model accuracy \cite{shi2009hash}.
For us, on the contrary, a high number of collisions is desired, since it implies an increase in privacy:
$k$ features hashing to the same bin implies that we cannot distinguish between those $k$ features.
Lemmata~\ref{lemma:1}~and~\ref{lemma:2} show that non\hyp colliding features, which would undermine privacy, are highly unlikely. We can, e.g., guarantee a high minimum number of collisions for each feature with overwhelming probability.

A malicious server could try to choose the hash function in such a way that certain features do not collide and can thus be distinguished.
This can, however, easily be prevented by letting the hash depend on a trusted unpredictable random value, as, e.g., provided by NIST \shortcite{nist}.

To see why we want to hide feature values from the server, consider,
e.g., the case where the features are the frequencies with which strings occur in a text field into which a user types text:
here, the feature itself could already reveal sensitive information, such as email addresses.
Features could also help determine a user's label: if, e.g., each feature denotes whether or not a specific URL has been visited by the user, then the server could infer the label of the sole visitor $i$ of URL $j$
(e.g., for the URL \url{http://github.com/User123}, the sole visitor is most likely User123 themselves), by the
corresponding update vector entry $\delta(1-y_i w^T x_i) \, y_i x_{ij} = y_i$.

\xhdr{Hiding labels from the server: splitting update vectors}
In addition to keeping the features private, we also want to prevent the server from knowing the label of a user.
Recall that, during training, only the update vector \(\nabla_w L(w, x_i, y_i)\) (and not the label \(y_i\)) is sent to the server.
Nonetheless, in the case of a linear SVM with binary features, if one of the entries \(\delta(1-y_i w^T x_i) \, y_i x_{ij}\) of the update vector is non-zero, then that very entry equals---and thus reveals---the label.
If (and only if) the server knew to which user the update vector belongs, it would know this user's label.
Since by temporary assumption A1, the server is given the update vectors only, it would have to identify the user via the update vector, which we must hence prevent.
To do so, we use a property of the subgradient update (\Eqnref{eq:general_update}):
not only is one user's local update independent of the other users',
but also the update for one entry \(w_j\) does not rely on the update for any other entries, which means we can update each entry individually.
We exploit this fact by splitting the update vector
$\nabla_w L(w, x_i, y_i)$
into its individual entries and sending each entry
\(\delta(1-y_i w^T x_i) \, y_i x_{ij}\)
together with its index $j$ as a separate package.
In the case of binary \(x_{ij}\), \(\delta(1-y_i w^T x_i) \, y_i x_{ij}\) can only take on the values $-1$, $0$ and $1$, therefore making it impossible for the server to determine which packages stem from the same feature vector, since they cannot be discerned from other clients' packages.
The case of binary features can easily be extended to integer features: instead of sending one package containing \(y_i x_{ij}\),
the client may send \(\abs{x_{ij}}\) packages containing \(y_i\sgn(x_{ij})\).
(Note that packages where \(\delta(1-y_i w^T x_i)=0\) need not be sent.)

Since, after this change to the protocol, the server only receives packages containing 1 or $-1$ and the respective feature index,
this is equivalent to knowing the number $N_j^+$ of positive packages and the number  $N_j^-$ of negative packages received for each feature index $j$.
As mentioned before, \(\delta(1-y_i w^T x_i) \, y_i x_{ij} \in \{0, y_i\}\), i.e., only users with label 1 contribute to  $N_j^+$ and only users with label $-1$ contribute to  $N_j^-$.
Determining the label of a user is thus equivalent to determining
whether the user's update vector contributed to $N_j^+$ or to $N_j^-$.
The confidence with which this can be done is vanishingly small, as we show in Lemma \ref{lemma:3}, even for the (from a privacy perspective worst) case that the server knows all of a user's features from some external source (the maximum a server that does not already know the label could possibly know).

\xhdr{Dropping temporary assumption A1}
We now drop the assumption that the server receives the update data as a single database. In the new setting, the server additionally receives
(1)~the IP address from which a package was sent and
(2)~the time at which a package arrives.
We need to make this information useless in order for the privacy guarantees from above to still hold, as follows.
First,
we route all messages through an anonymity network, such as \textit{Tor} \cite{syverson2004tor}, or through a similar proxy server infrastructure, such as the one we use in our online experiment (\Secref{sec:cliqz}), thereby removing all information that was originally carried by the IP address.
Without this measure, the server could directly link all packages to the client that sent them, thus effectively undoing the above\hyp described splitting of feature vectors.
Second, to remove information that the arrival times of packages might contain, we set the length of one training iteration to \(n\) seconds and make clients send their packages not all at once, but spread them randomly over the \(n\) seconds, thus making the packages' arrival times useless for attacks.
Without this measure, all update packages from the same user might be sent right after one another, and the server would receive groups of packages with larger breaks after each group and could thus assume that each such group contains all packages from exactly one user.

\xhdr{Dropping temporary assumption A2}
Finally, we drop the assumption that the server honestly follows the training procedure by sending the same correctly updated weight vector to all clients in each iteration.
In order to not depend on this assumption, we give clients a way to recognize when the server violates it:
Instead of requesting the training data in an iteration once from the server, the clients request it multiple times.
Only if they get the same response every time do they respond; otherwise they must assume an attack.
Since the clients' requests are routed through an anonymization network, the server cannot identify subsequent request from the same client and cannot maliciously send the same spurious weight vector every time. To reduce communication costs, the clients don't actually request the training data multiple times, but only once in the beginning, and afterwards request hashes of it.
As an even safer countermeasure, one could distribute weight vectors and all auxiliary information via a blockchain. This way, each user would be able to verify the integrity of the data they receive.

In the case that one does not prevent the server from distributing different data to different clients, the server could, e.g., distribute an all-zero weight vector
\(w=0\in\R^{\tilde{d}}\)
to all users in a first step.
All of them would then respond with a non-zero update revealing all of their non-zero features, since then \(1-y_i w^T x_i = 1\) for arbitrary \(y_i\) and \(x_i\). In the next step, the server would send out the zero weight vector to all but one user \(\ell\). This user would instead receive the weight vector \(e_1=(1,0,\dots,0)\). If \(y_\ell x_{\ell 1} \leq -1\), then user \(\ell\) would in this round not send back any updates. Otherwise the server would do the same with \(-e_1\), then with \(e_2\), \(-e_2\), and so on, until it receives fewer updates than in the last iteration.
It would then know that all the missing updates come from one user, and would thus be able to reconstruct this user's label and feature vector (however, still hashed).
In a similar fashion, the server could undermine single users' privacy by manipulating the time they have for sending their updates, \ie, the length of the training iteration. All users would get a very long deadline, and one user \(\ell\) a very short one. Then the packages arriving before the short deadline would mostly come from user \(\ell\).

\xhdr{SecVM: model and protocol}
To conclude, we summarize the final SecVM model.
As an initial step, all clients hash their feature vectors into a lower\hyp dimensional space.
    Then the training procedure begins.
    The server sends out the current parameter vector \(w\) to all clients.
    Each client \(i\) computes its local update \(g_{i} := \nabla_w L(w, x_i, y_i) = \delta(1-y_i w^T x_i) \, y_i x_i\) and splits this vector into its individual entries \(g_{ij}\).
    These entries, together with their indices \(j\), are sent back to the server as individual packages at random points in time via a proxy network.
    The server sums up the updates \(g_{ij}\) corresponding to the same entry \(j\) of the parameter vector and updates the weight vector \(w\) accordingly.
    This procedure is repeated until the parameter vector has converged.
This model meets the privacy requirement from \Secref{sec:problem statement}, as it does so under temporary assumptions A1 and A2, and, as we have shown, we may drop these assumptions without violating the privacy requirement.

\section{Offline evaluation: gender inference for Twitter users}
\label{sec:twitter}
We implemented and tested our approach in a real application with users connected via the Internet, as described in \Secref{sec:cliqz}. However, to assess its feasibility and to determine suitable hyperparameters, we first worked with offline data.
An and Weber \cite{AnW16} generously provided us with a dataset containing tweets collected from nearly 350,000 Twitter users alongside demographic data inferred from the users' profiles. Around half of the users are labeled as male, and half of them as female.

As the classification task for the SVM, we decided to predict the gender of users from the words they used in their tweets. The feature space is therefore very high\hyp dimensional; after preprocessing, we obtained
96M
distinct words. On the other hand, the feature vectors are very sparse, with only 1,826 words per user on average.
The properties of the dataset and the classification task are very similar to what we would expect to encounter in our later online experiment, so we deemed this a good initial evaluation of our method.

We only compare SecVM --- using different levels of hashing --- with a vanilla SVM that doesn't use hashing, and not with other classification methods (logistic regression, neural networks etc.). This is due to the fact that their typically real\hyp valued gradient entries may reveal private information, which makes those methods unsuited for fulfilling our privacy requirement (see also \Secref{sec:discussion}). Therefore the choice is not between training an SVM or training a different classifier, but rather between training an SVM or not being able to train a classifier at all due to the privacy constraints. Or, to put it differently: our goal is not to achieve a classification performance that can compete with non\hyp privacy\hyp preserving methods, but to learn a useful classifier in our restricted setting.

\xhdr{Methodology}
For subgradient methods, one has to decide on a step size \(\eta\). We chose the step size of the Pegasos algorithm \cite{Shalev-Shwartz2011}, \ie, \(\eta = 1/\lambda t\), where \(t\) is the current iteration and \(\lambda\) the regularization parameter;
but whereas Pegasos performs stochastic subgradient training, we used regular subgradient training, where all training samples, rather than a subset, are used to compute each update.

For training and testing an unhashed SVM, we first split the dataset into two equally large sets \(A\) and \(B\), and each of those sets into 90\% training and 10\% test data. Then we trained an SVM with different values of \(\lambda\) on \(A\)'s training set, chose the one with the best accuracy on \(A\)'s test set, and did another round of training on \(B\)'s training set with this \(\lambda\).

\begin{figure*}
    \centering
    \begin{subfigure}[t]{0.49\textwidth}
        \includegraphics[width=\textwidth]{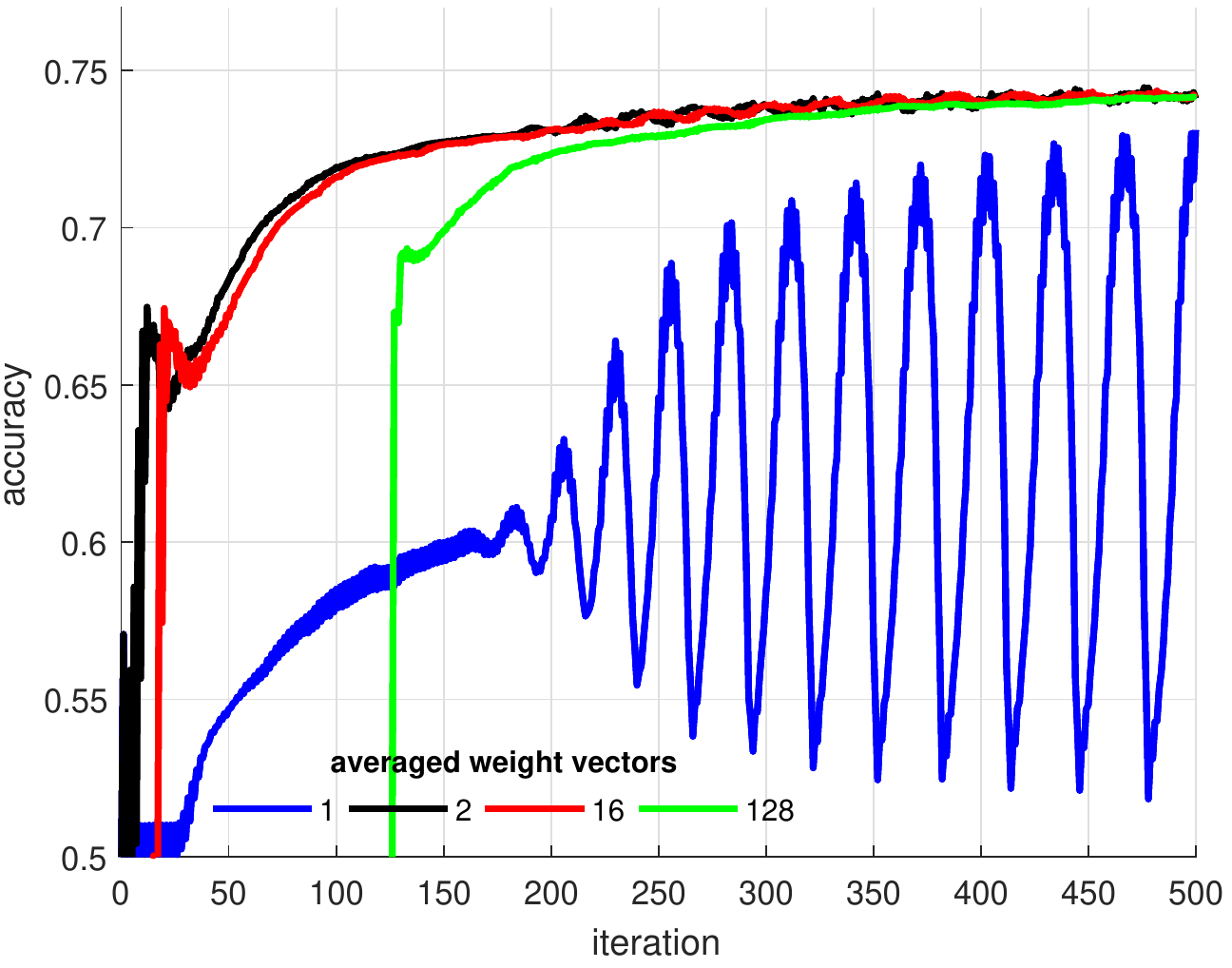}
    	\caption{No feature hashing used;
    	for various numbers of averaged weight vectors (\cf\ \Secref{sec:twitter}), each as one curve.}
    	\label{fig:twitter_unhashed}
    \end{subfigure}
    \hfill
    \begin{subfigure}[t]{0.49\textwidth}
        \includegraphics[width=\textwidth]{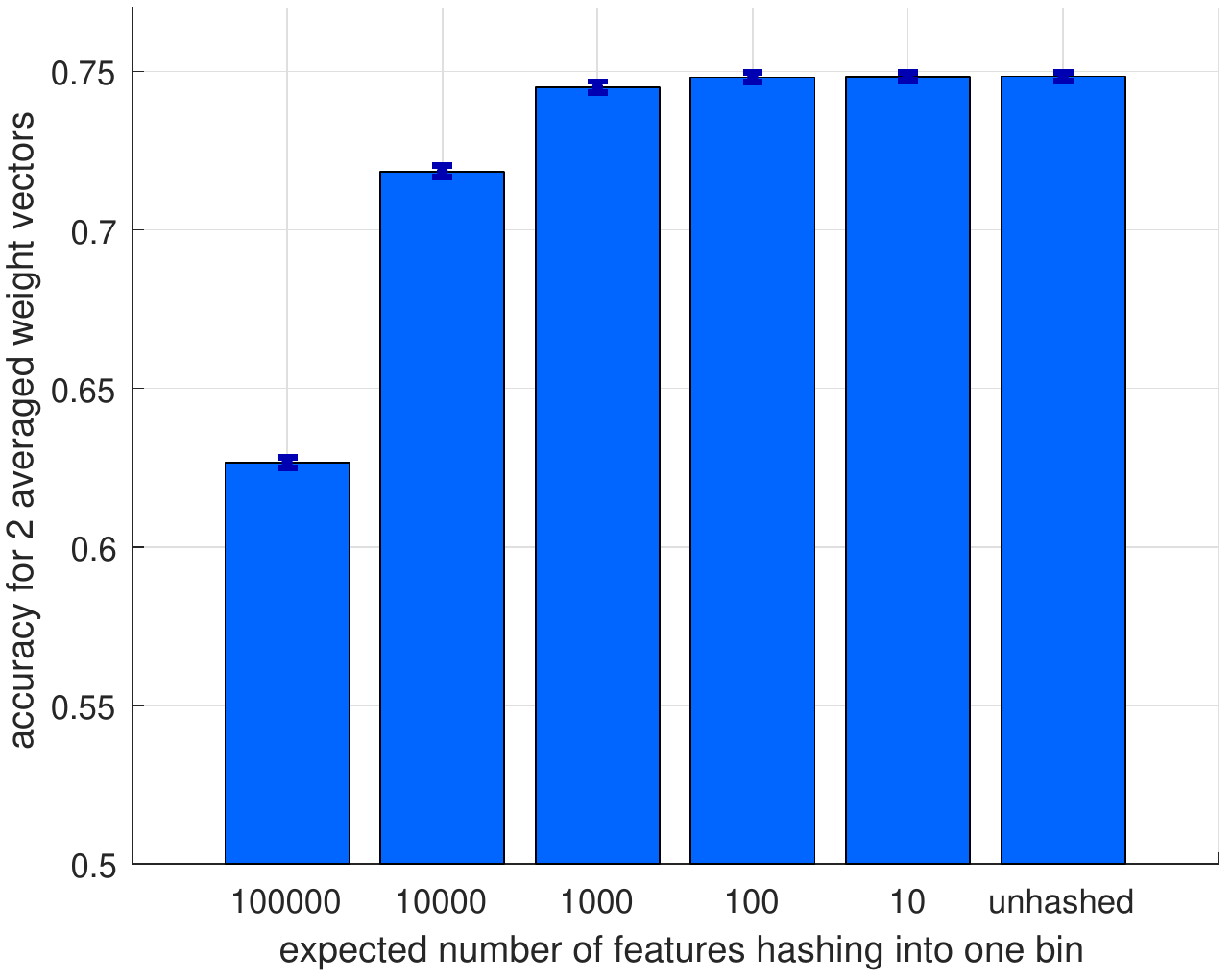}
    	\caption{Accuracy achieved for different numbers of hash bins
    	(averaging two weight vectors; cf.\ \Figref{fig:twitter_unhashed}).
    	}
    	\label{fig:twitter_bin_numbers}
    \end{subfigure}
    \caption{Offline evaluation on Twitter gender prediction task (\Secref{sec:twitter}).}
\end{figure*}

\xhdr{Results}
\Figref{fig:twitter_unhashed} reports the performance of this second SVM on \(B\)'s test set. Choosing any \(\lambda\) between \(10^{-7}\) and \(10^{-2}\) only made a marginal difference in accuracy.
The subgradient method is not a descent algorithm, \ie, it does not decrease the value of the objective function in each step. The effect of this can be seen in the blue curve of \Figref{fig:twitter_unhashed}:
test accuracy fluctuates a lot.
Therefore, inspired by Corollary~1 of the Pegasos paper \cite{Shalev-Shwartz2011}, we averaged the weight vector of subsequent iterations and investigated how the number of averaged weight vectors affects test performance. \Figref{fig:twitter_unhashed} shows that averaging only two weight vectors already gives not only an almost monotonic learning curve, but also a prediction accuracy strictly above that achieved when not using averaging.
For the following results, we thus always averaged two subsequent weight vectors, as averaging more vectors only led to slower convergence.
We obtained an accuracy of 75.2\% on the original, unhashed data.

To evaluate the influence of hashing on accuracy, we did 141 random splits of the entire dataset into 90\% training and 10\% test, with fixed \(\lambda=10^{-4}\). \Figref{fig:twitter_bin_numbers} shows means and standard deviations of the fraction of correctly predicted labels for different numbers of hash bins.
For instance, when reducing the dimension of the original feature and weight vectors by a factor of 1,000 (\ie, hashing the 96M original features into 96K bins), accuracy only drops from 75.2\% to 74.7\%, showing that even very aggressive hashing has only a negligible impact on the SVM's performance.

\section{Online evaluation: gender inference for Web surfers}
\label{sec:cliqz}
In addition to the above offline evaluation (\Secref{sec:twitter}), we also tested our method in its intended environment: as part of a software that runs on user\hyp held, Internet\hyp connected devices, in this case the \Cliqz{} web browser
\cite{cliqz}.
Via a collaboration, we deployed SecVM to a sample of \Cliqz{}'s user base of more than 200K daily active users.
For data collection, \Cliqz{} uses the idea of client- instead of server\hyp side data aggregation via a framework called
\HumanWeb
and based on a proxy network
\cite{humanweb, modi}.
The work presented in this paper leverages some of the concepts introduced by this framework; but SecVM goes one step further and not only does data aggregation on the client side, but also computations on this data.

\xhdr{Task} As in the offline evaluation (\Secref{sec:twitter}), we decided to build an SVM that can predict a user's gender, but this time not from the words they type, but rather from the words contained in the titles of the websites they have visited. This setting is much more challenging than the one of \Secref{sec:twitter}: users go on- and offline at unpredictable points in time, the number of available users varies heavily between the different times of day, and together with the set of training users, the set of test users changes, too, thus giving less accurate assessments of prediction performance.

\xhdr{Implementation}%
\footnote{Source code for client available at \url{https://github.com/cliqz-oss/browser-core/tree/6945afff7be667ed74b0b7476195678262120baf/modules/secvm/sources}.
Source code for server available at \url{https://github.com/epfl-dlab/secvm-server}.}
We implemented the client side as a module in the \Cliqz{} browser.
It extracts the features (words contained in website titles) from a user's history, and the label (gender) from the HTML code of \url{www.facebook.com} once a user visits that page and logs in.
The clients regularly fetch a static file from a server that describes the experiments that are currently running: the number of features to be used, the current weight vector, etc.
Apart from this, it contains the percentage \(p\) of users that should participate in training, while the others are used for testing.
The first time a user encounters a new experiment, they assign themselves to training with probability \(p\), and to test with probability \(1-p\) (we chose $p=0.7$).
The file fetched from the server also informs the user how much time they have left until the end of the current training iteration to send their data back to the server.
This data consists either of the update packages or a single test package, depending on whether the user is part of the training or the test set.
To avoid temporal correlation attacks, each package is sent at a random point in time between its creation and the end of the training iteration.
IP addresses are hidden from the server by routing the packages through the \HumanWeb proxy network,
\cite{humanweb},
a system that is functionally similar to Tor and prevents identification of users at the network level.\footnote{The technical report describing the \Cliqz{} Human Web \cite{humanweb} contains a discussion on why TOR was not suitable for the \Cliqz{} use case.}

\begin{figure}
    \begin{subfigure}[t]{0.67\columnwidth}
        \includegraphics[width=\textwidth]{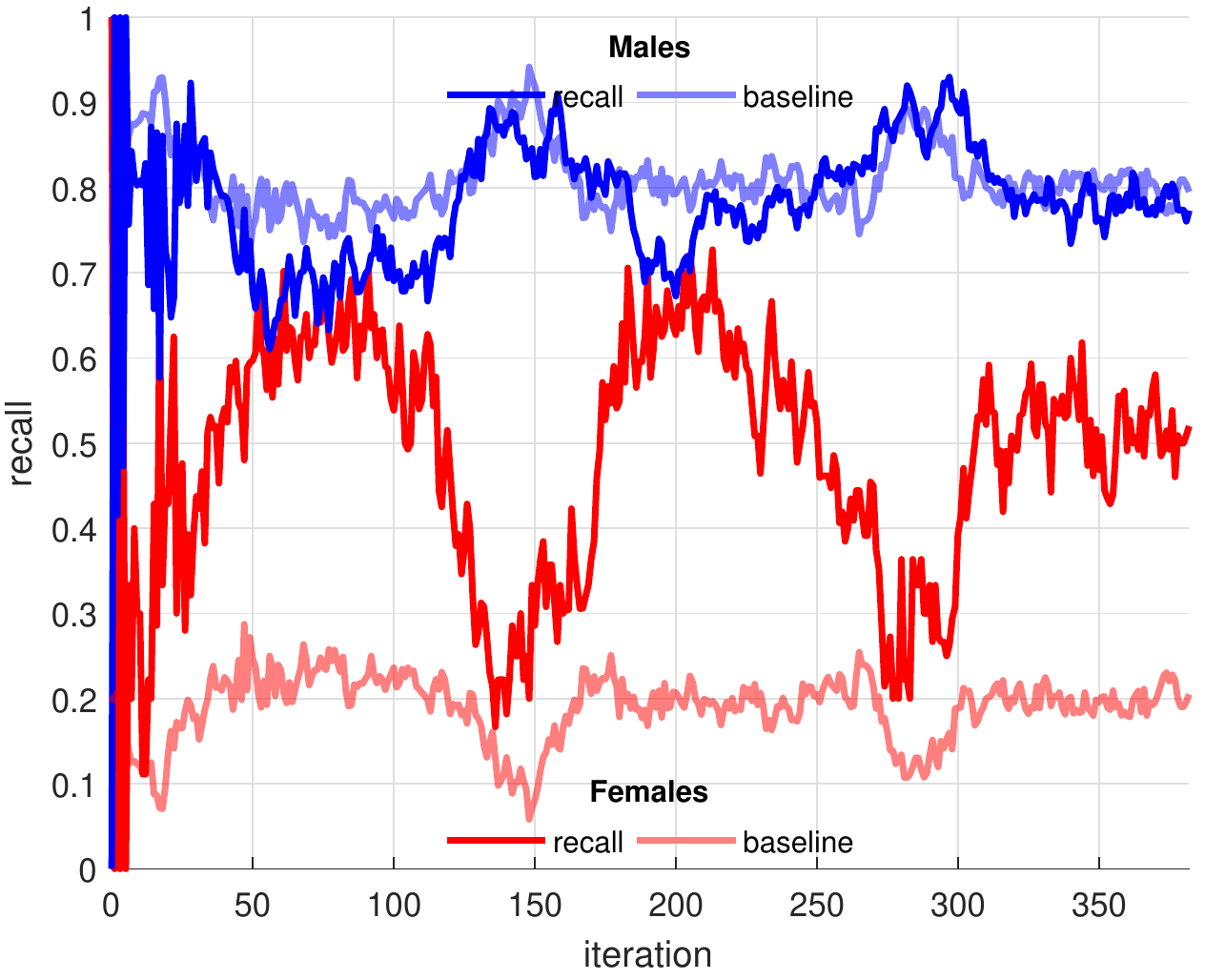}
    	\caption{
    	\textit{Solid red (blue):} \textbf{recall} for females (males).
    	\textit{Transparent red (blue):} marginal baseline for females (males), i.e., share of females (males) in test set.
    	}
        \label{fig:cliqz_10000_bins}
    \end{subfigure}
    \hspace{0.01\textwidth}
    \begin{subfigure}[t]{0.26\columnwidth}
        \includegraphics[width=\textwidth]{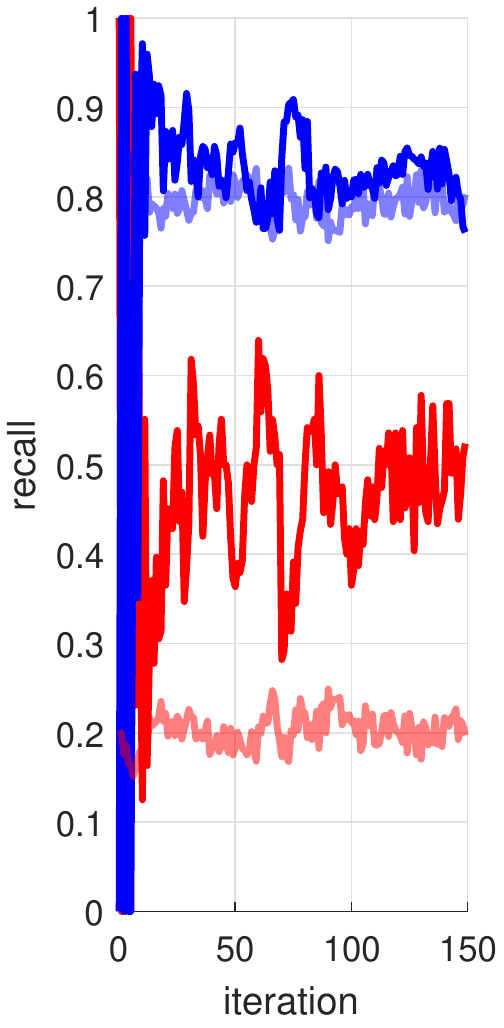}
    	\caption{Same as \Figref{fig:cliqz_10000_bins}, but only trained during daytime.
    	}
        \label{fig:cliqz_10000_bins_day}
    \end{subfigure}
    \begin{subfigure}[t]{0.67\columnwidth}
        \includegraphics[width=\textwidth]{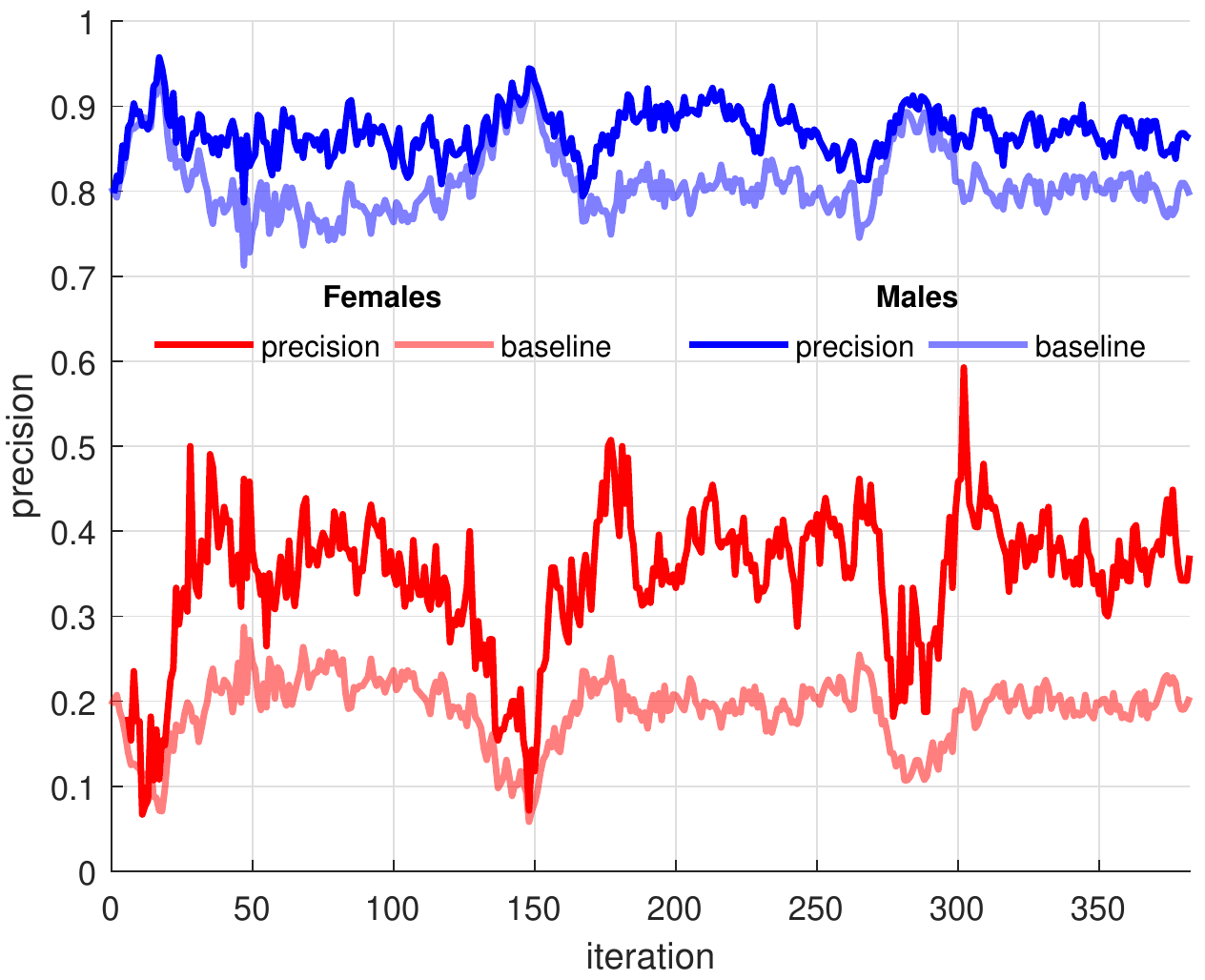}
    	\caption{
    	\textit{Solid red (blue):} \textbf{precision} for females (males).
    	\textit{Transparent red (blue):} marginal baseline for females (males), i.e., share of females (males) in test set.
    	}
        \label{fig:cliqz_10000_bins_prec}
    \end{subfigure}
    \hspace{0.01\textwidth}
    \begin{subfigure}[t]{0.26\columnwidth}
        \includegraphics[width=\textwidth]{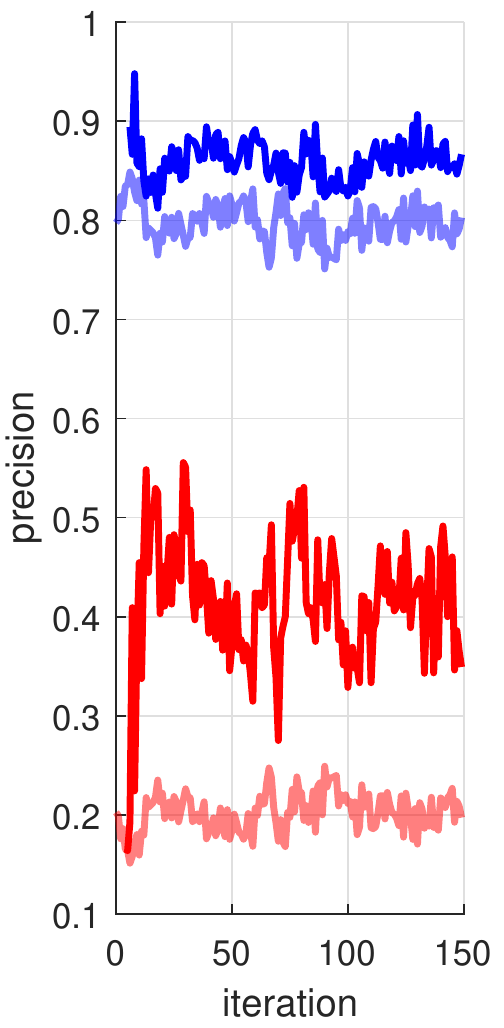}
    	\caption{Same as \Figref{fig:cliqz_10000_bins_prec}, but only trained during daytime.
    	}
        \label{fig:cliqz_10000_bins_day_prec}
    \end{subfigure}
    \hspace{0.01\textwidth}
    \caption{Online evaluation on task of predicting web surfers' gender (\Secref{sec:cliqz}). We chose to perform one update every 11 minutes, so 50 updates took about 9 hours.}
    \label{fig:cliqz}
\end{figure}

\begin{figure}
        \includegraphics[width=\columnwidth]{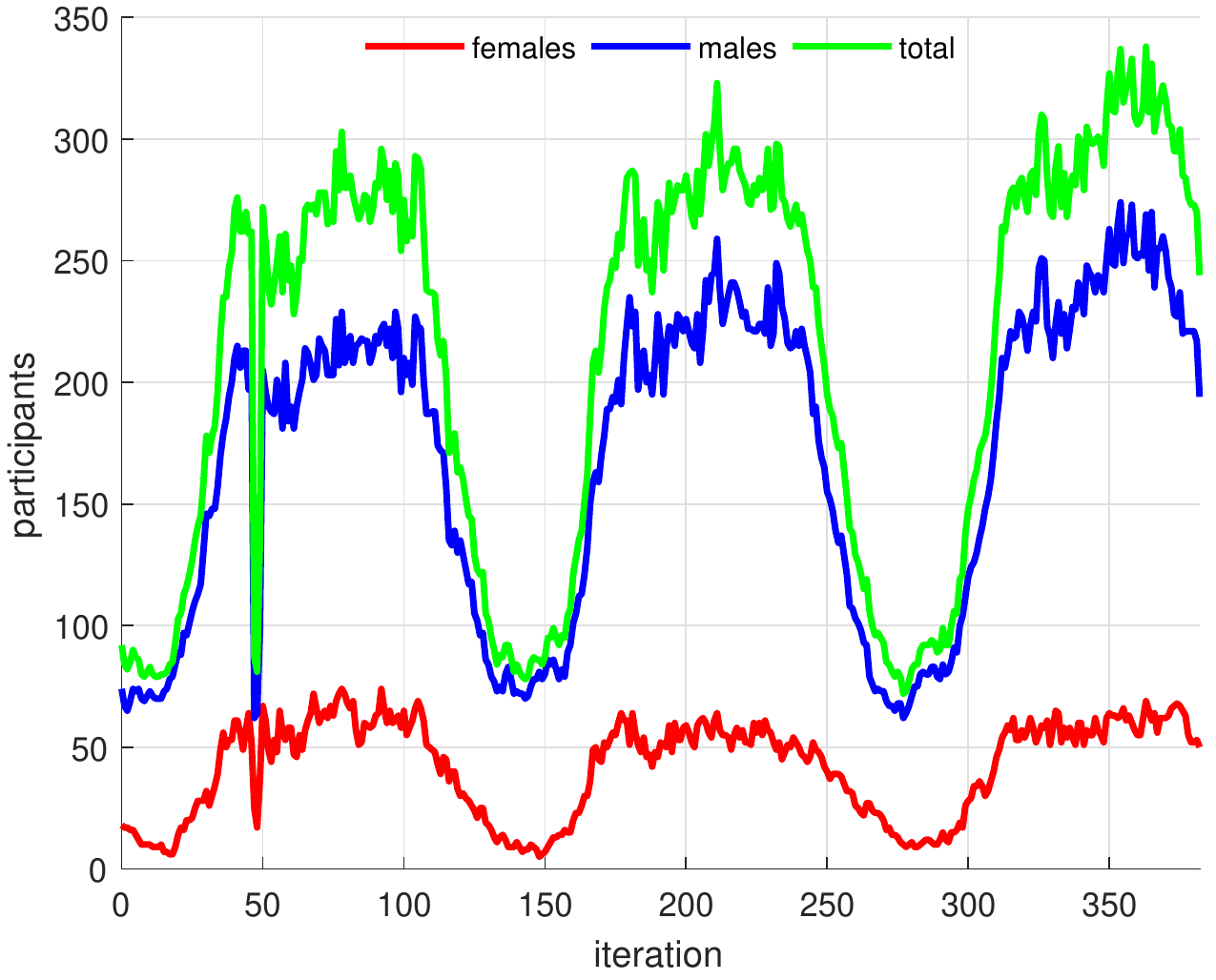}
    	\caption{Online evaluation (\Secref{sec:cliqz}): absolute number of test participants (30\% of all participants) during the training shown in \Figref{fig:cliqz_10000_bins} and \ref{fig:cliqz_10000_bins_prec}.
    	\textit{Red:} females.
    	\textit{Blue:} males.
    	\textit{Green:} total.
    	}
        \label{fig:participants}
\end{figure}

\xhdr{Results}
We performed two runs of the experiment, one for three days during night- and daytime (\Figref{fig:cliqz_10000_bins}), and one for two days only during daytime (\Figref{fig:cliqz_10000_bins_day}). As parameters, we chose \(\lambda=0.1\) as the regularization parameter, 10,000 bins for feature hashing, and, as in the offline experiment on the Twitter dataset (\Secref{sec:twitter}), averaged two subsequent weight vectors.
As can be seen in \Figref{fig:participants}, there were a lot fewer females in the user base, which is why we weighed their updates by a factor of 4.
We set each training iteration to last 11 minutes,
in order to minimize the load on the client side.

In each iteration, we let the users from the test set evaluate the current model.
We evaluate performance in terms of recall for males and females, i.e., the fraction of all males and females, respectively, that the classifier predicts correctly.
Results are shown in
\Figref{fig:cliqz}a--b.
The dark solid lines show recall for female and male users, respectively, the transparent lines their respective share in the test set.
This share is equal to the recall that would be achieved by a baseline classifier randomly guessing each class with its marginal probability (i.e., between 10\% and 20\% for females, depending on the iteration number, and between 80\% and 90\% for males).
We see that recall for males is at around the same (high) level as the baseline, while recall for females (the minority class) lies far above this baseline.
However, despite the weighting, recall for females is still significantly worse than recall for males.
We attribute this to the fact that the number of training samples was much smaller for this class.

Noteworthy are the drops for females and spikes for males around iterations 150 and 300 of the continuous training (\Figref{fig:cliqz_10000_bins}).
They coincide with nighttime and a drop in overall active users, as shown in \Figref{fig:participants}.
At night, the share of female users in the overall population drops to about half, compared to daytime (10\% vs.\ 20\%), and we explain the low female recall during those times by the fact that we see very little training signal for females then.
Hence, one would preferably stop training during the night, which we did in our second run, whose results (\Figref{fig:cliqz_10000_bins_day}) are much more stable.

Besides recall, we also evaluated precision.
It lies far above the baseline for both females and males (\Figref{fig:cliqz}c--d).

\section{Discussion}
\label{sec:discussion}

\xhdr{Other models beyond SVM}
Linear SVMs are among the most popular classification algorithms in research as well as the industry across domains \cite{1757-899X-436-1-012020,CHEN2017314,Liu2017,10.1371/journal.pone.0161501, 7755785, PMID:27919211}.
Their popularity is, however, not the only reason why we chose SVMs as the machine learning model underlying our framework.
The SVM model also has the attractive property that the partial derivatives of the user\hyp dependent part of its loss function can only take on the values $-1$, 0, and 1, which we exploit to make the reconstruction of labels impossible.
To illustrate this advantage, consider what would happen when using, e.g., logistic regression instead.
In this case, we have \(L(w, x_i, y_i) = \log(1 + \exp(-y_i w^T x_i))\) and
\begin{equation}
\label{eqn:logreg}
    \frac{\partial L(w, x_i, y_i)}{\partial w_j} = \frac{y_i x_{ij}}{1 + \exp(y_i w^T x_i)}.
\end{equation}
When all feature values are integers, the numerator of \Eqnref{eqn:logreg} can only take on integer values, allowing the real\hyp valued denominator to uniquely identify user $i$, thus making it easy to associate updates from user \(i\) for different values of \(j\) with each other, even when splitting update vectors into their individual components, as introduced in \Secref{sec:solution} precisely as a counter\hyp measure against such record linkage.
However, we point out that rectifiers---the same function as the SVM loss---are a popular choice as an activation function for neural networks \cite{glorot2011deep}, and there is also research on using binary weights in neural networks \cite{courbariaux2015binaryconnect}.
Both of these observations open avenues to extend our work to neural networks.

\xhdr{Practical considerations} As opposed to federated learning \cite{mcmahan2016federated}, we split the updates from individual users into atomic pieces. Only this allows us to obtain such strong privacy guarantees; but, of course, it comes at the expense of efficiency: due to protocol overhead, sending many small packages instead of one big package produces more traffic both on the client and on the server side, which in turn slows down training, especially if one has no control over the client devices and has to be conservative when estimating the speed of their network connection in order not to congest it. By applying our method in a real\hyp world production setting (\Secref{sec:cliqz}), we showed that this overhead is very much manageable. 

As opposed to the privacy\hyp focused extensions of federated learning \cite{geyer2017differentially,mcmahan2017learning}, SecVM works without adding random noise.
Especially when feature values are mostly small integers,
such noise needs to be large in order to achieve a sufficient level of privacy.
Apart from that, our goal is also a different one.
The aforementioned methods 
\cite{geyer2017differentially,mcmahan2017learning} offer differential privacy guarantees against attacks that aim at exploiting the differences between different iterations of the machine learning model during the training process that could be used by malicious clients to extract other \textit{clients'} data, whereas the server is assumed to be trusted.
SecVM, on the other hand, protects against a malicious \textit{server}.

Often the goal of privacy research in the field of machine learning is to prevent the extraction of sensitive information from the final model. Our approach starts even earlier, by preventing the extraction of sensitive information during the training. Of course the model itself that results from such a training procedure preserves the privacy of the training data as well.

\xhdr{Limitations}
First, we would like to emphasize that for our privacy guarantees to hold, the client module source code must be visible to the clients, as is the case for our experiment in \Secref{sec:cliqz}.

There are two cases in which SecVM can fail to protect privacy. The first one is when being unlucky with the hashing function. When a sensitive feature does not get hashed into the same bin as any other feature, then the server will learn about the presence of this feature. However, the probability of this happening is cryptographically small, as we show in \Appref{sec:app_hashing}.
The second case in which SecVM's privacy protection fails is when being unlucky with the data. If all gradients except from a single one are zero, which implies that the SVM manages to correctly classify all training samples except from this single one, and additionally the \(l_1\text{-norm}\) of all feature vectors is the same, then the server knows that all messages it receives stem from the same user --- the same holds true if only a single user participates in the training. If the server knows the feature values of this user from some other source, it will be able to reconstruct the user's label from the sign of the received messages. But due to the hashing, it will not be able to reconstruct any features. Note also that the event that an SVM achieves an almost perfect classification accuracy is unlikely on real\hyp world datasets due to the limited model complexity.

A different attack vector is via weaknesses of anonymization networks, but this is outside of the scope of this paper; as well as malicious clients trying to poison the training by sending manipulated gradients.

The datasets that we experimented with are only high\hyp dimensional datasets. On these datasets, the hashing method is particularly useful, because it doubles as a method for dimensionality reduction. On small\hyp dimensional datasets, hashing might reduce the dimension to a value so small that the SVM's classification performance drops significantly. However, on typical low\hyp dimensional datasets (think of, e.g., demographic data), hashing will not even be necessary, because every feature is non\hyp zero for a large number of users and thus no single feature is identifying.

\section{Conclusion}
\label{sec:conclusion}
We proposed SecVM, a framework for training an SVM on an \textit{untrusted server} based on distributed data
while preserving the data owners' privacy, without the need for direct access to their data.
Instead, data owners only send model updates to the training server.
The system relies on three main ideas to prevent the reconstruction of sensitive information from those updates:
routing all messages through a proxy network,
splitting the update vector into its components,
and hashing the original features.
We implemented SecVM in the \Cliqz{} browser---a real distributed system---with promising results.

\appendix

\section{Privacy Guarantees}
\label{sec:privacy}

\subsection{Hashing}
\label{sec:app_hashing}

\begin{notation}
Let \(m\) denote the number of unique features and \(n\) the number of bins we hash them into.
The hashing is executed by a function \(h\) drawn uniformly at random from the family of all hash functions from the set of strings to \([n]\), where by \([n]\) we denote the set \(\{1,\dots,n\}\). Probabilities are taken over the random choice of \(h\).
\end{notation}

We give bounds on 3 probabilities that all can be used to determine the number of hash bins to use for a desired level of privacy.

\begin{lemma}
\label{lemma:1}
\begin{enumerate}[label=(\alph*)]
    \item Let \(p_1\) be the probability that there exists at least one feature which does not collide with any other feature. Then
    \begin{equation*}
        p_1 \leq m\left(\frac{n-1}{n}\right)^{m-1}.
    \end{equation*}
    \item Let \(K\subseteq [m]\) be a set of specific features, \(k\defeq |K|\). For the probability \(p_2\) that at least one of the features in \(K\) does not collide with any other feature we have that
    \begin{equation*}
        p_2 \leq k \left(\frac{n-1}{n}\right)^{m-1}.
    \end{equation*}
\end{enumerate}
\end{lemma}

\begin{proof}
\begin{enumerate}[label=(\alph*)]
    \item Using a union bound in the first step, we get
    \begin{align*}
        p_1 &\leq \sum_{i=1}^n\Pr[\text{exactly one features hashes into bin }i]\\
        &= \sum_{i=1}^n \sum_{j=1}^m \frac{1}{n} \left(\frac{n-1}{n}\right)^{m-1}\\
        &=m\left(\frac{n-1}{n}\right)^{m-1}.
    \end{align*}
    In the second line we summed over the probabilities for a specific one of the \(m\) features to end up in bin \(i\).
    \item The statement again follows from a union bound argument:
    \begin{align*}
        p_2 &\leq \sum_{i\in K} \Pr[\text{feature \(k\) does not collide}]\\
        &= k \left(\frac{n-1}{n}\right)^{m-1}
    \end{align*}
\end{enumerate}
\end{proof}

\begin{lemma}
\label{lemma:2}
Let \(p_3\) be the probability that each feature collides with at least \(k-1\) other features. Assume that \(k\leq m/n\) (otherwise \(p_3=0\)).
Then
\begin{equation*}
    p_3 \geq 1 - \binom{m}{k-1} \frac{(n-1)^{m-k+1}}{n^{m-1}} \frac{m-k+2}{m-n k +n+1}.
\end{equation*}
\end{lemma}

\begin{proof}
Note first that
\begin{align*}
    p_3 &\geq \Pr[\text{at least \(k\) features per bin}]\\
    &= 1 - \Pr[\text{at least one bin with less than \(k\) features}],
\end{align*}
since we exclude the possibility of having empty bins. Then by a union bound
\begin{align*}
    &1 - \Pr[\text{at least one bin with less than \(k\) features}]\\
    &\geq 1 - \sum_{i=1}^n \Pr[\text{\(<k\) features in bin \(i\)}]\\
    &= 1 - \sum_{i=1}^n \sum_{l=0}^{k-1} \Pr[\text{exactly \(l\) features in bin \(i\)}]\\
    &= 1 - \sum_{i=1}^n \sum_{l=0}^{k-1} \sum_{\substack{J\subseteq [m]\\ |J| = l}}\Pr[\text{exactly the features in \(J\) in bin \(i\)}]\\
    &= 1 - \sum_{i=1}^n \sum_{l=0}^{k-1} \sum_{\substack{J\subseteq [m]\\ |J| = l}}\left(\frac{1}{n}\right)^l \left(\frac{n-1}{n}\right)^{m-l}\\
    &= 1 - \sum_{i=1}^n \sum_{l=0}^{k-1} \binom{m}{l} \frac{(n-1)^{m-l}}{n^m}\\
    &= 1 - n \sum_{l=0}^{k-1} \binom{m}{l} \frac{(n-1)^{m-l}}{n^m}\\
    &= 1 - n \left(\frac{n-1}{n}\right)^m \sum_{l=0}^{k-1} \binom{m}{l} \left(\frac{1}{n-1}\right)^l.
\end{align*}
To bound the sum, we adapt a proof of \cite{mathoverflow:binomial}. We observe that
\begin{align*}
    &\frac{\binom{m}{k-1} \left(\frac{1}{n-1}\right)^{k-1} + \binom{m}{k-2} \left(\frac{1}{n-1}\right)^{k-2} + \binom{m}{k-3} \left(\frac{1}{n-1}\right)^{k-3} + \dots}{\binom{m}{k-1} \left(\frac{1}{n-1}\right)^{k-1}}\\
    &= 1 + (n-1)\frac{k-1}{m-k+2}\\
    &+ (n-1)^2 \frac{(k-1)(k-2)}{(m-k+2)(m-k+3)} + \dots,
\end{align*}
which can be bounded from above by the geometric series
\begin{align*}
    &1 + (n-1)\frac{k-1}{m-k+2} + \left((n-1)\frac{k-1}{m-k+2}\right)^2 + \dots\\
    &= \sum_{l=0}^\infty \left((n-1) \frac{k-1}{m-k+2}\right)^l\\
    &= \frac{m-k+2}{m - n k + n + 1}.
\end{align*}
The series converges because \(k\leq m/n\). This calculation yields
\begin{align*}
    &1 - n \left(\frac{n-1}{n}\right)^m \sum_{l=0}^{k-1} \binom{m}{l} \left(\frac{1}{n-1}\right)^l\\
    &\geq 1 - n \left(\frac{n-1}{n}\right)^m \binom{m}{k-1} \left(\frac{1}{n-1}\right)^{k-1} \frac{m-k+2}{m - n k + n + 1}\\
    &= 1 - \binom{m}{k-1} \frac{(n-1)^{m-k+1}}{n^{m-1}} \frac{m-k+2}{m-n k +n+1}.
\end{align*}
\end{proof}

To get some intuition for these quantities, we give an example.
In the Twitter experiment in Sec. 5 with \(m=95,880,008\) we saw no significant decrease in prediction accuracy for \(n=95,880\). For these values, we get \(p_1 < 5\times 10^{-427}\) and \(p_2 < 6 k \times 10^{-435}\). The \(k\) in Lemma \ref{lemma:2} can be chosen between 1 and 1,000. For \(k=700\), we have \(1-p_3<5\times 10^{-19}\).

In \cite{shi2009hash} it was proved that hashing each feature into multiple bins can increase the probability that for each feature there exists at least one bin where it does not collide with any other feature. Of course this only holds if the number of bins is higher than the number of features; also we don't want features that have no collisions. Nevertheless, this is an interesting option even in our case. Assume that the number of features is higher than the number of bins, but that many of them are not very indicative of the label we are trying to predict. Thus we are only interested in preventing collisions between indicative features. If we set \(n\) in Theorem 1 of \cite{shi2009hash} to be the number of indicative features, we obtain a bound for their non-collision probability if we hash all features into multiple bins. In our experiments in Sec. 5 and Sec. 6 we are in a situation where most features are not very indicative --- however, increasing the number of bins each feature is hashed into slightly decreased the accuracy instead of increasing it.

\subsection{Splitting}

As shown in \Secref{sec:solution} of the paper, determining whether a user's label is 1 (-1) is equivalent to determining whether they have contributed to the sum of positive (negative) update vectors. For this task, we allow the server the maximal knowledge, i.e., the knowledge of the entire feature vector, which equals the update vector up to multiplication with -1, 0 or 1.

We formalize the task of the server as the discrimination between two worlds. In world 1, the vectors of all users are random. In world 2, the vectors of all users except from the one to be attacked are random; the vector of the latter one is known to the server.

\begin{lemma}
\label{lemma:3}
Let \(u_j\) denote the \(j\)\hyp th entry of a vector \(u\in \N^d\). Let \(M > 1\) be the number of users participating in the training, \(d>0\) the dimension of the update vectors and \(F>0\) the (fixed) \(\ell_1\)\hyp norm of the update vectors. The model is trained for \(K>0\) iterations.\\
Let \(X^{mkf} \sim U(u\in \N_{\geq 0}^d:\ \text{There exists exactly one }j^*\in \{1,\dots,d\}\allowbreak \text{s.t. }\allowbreak u_{j^*}=1 \text{ and } u_j = 0 \text{ for all } j\neq j')\) be i.i.d. random vectors following the uniform distribution over all one\hyp hot vectors, where \(m \in \{1,\dots,M\},\ k\in\{1,\dots,K\},\ f\in\{1,\dots,F\}\). The update vector of the \(m\)\hyp th user in iteration \(k\) is then given by \(X^{mk}\defeq \sum_{f=1}^F X^{mkf}\). (Here we assume positive updates; the case of negative updates is analogous.) Let \(v^k\geq 0,\ \norm{v}_1 = F,\) be the (possible) update vector in the \(k\)\hyp th training iteration of the user to be attacked. Further let \(s^k\in\N_{\geq 0}^d,\ \norm{s^k}_1=MF,\) be the sum that the server receives in the \(k\)\hyp th iteration. Then
\begin{align*}
    \abs{&\Pr[\sum_{m=1}^M X^{mk} = s^k\quad \text{for all } k=1\dots K]\\
    &- \Pr[\sum_{m=1}^{M-1} X^{mk} + v^k = s^k\quad \text{for all } k=1\dots K]}\\
    \leq &\max\{p(M,F,d),\ p(M-1,F,d)\}^K,
\end{align*}
where
\begin{equation*}
    p(m,f,d)= \frac{(mf)!}{\left(\lfloor\frac{mf}{d}\rfloor!\right)^d} \frac{1}{d^{mf}}.
\end{equation*}
\end{lemma}

\begin{proof}
We show that both probabilities are bounded by the r.h.s. and then use that \(\abs{a-b}\leq\max\{a,\ b\}\) for non\hyp negative \(a\) and \(b\).
\begin{align*}
    &\Pr[\sum_{m=1}^M X^{mk} = s^k\quad \text{for all } k=1\dots K]\\
    &= \Pr[\sum_{m=1}^M X^{m1} = s^1] ^ K\\
    &= \Pr[\sum_{m=1, f=1}^{M, F} X^{m1f} = s^1] ^ K\\
    &= \Pr[\sum_{m=1, f=1}^{M, F} X_i^{m1f} = s_i^1\quad \text{for all } i=1\dots d] ^ K\\
    &= \left(\binom{MF}{s_1^1,\dots,s_d^1} \frac{1}{d^{MF}}\right)^K\\
    &\leq \left(\frac{(MF)!}{\left(\lfloor\frac{MF}{d}\rfloor!\right)^d} \frac{1}{d^{MF}}\right)^k\\
    &= p(M,F,d)^K
\end{align*}
For the second probability we similarly obtain
\begin{align*}
    &\Pr[\sum_{m=1}^{M-1} X^{mk} + v^k = s^k\quad \text{for all } k=1\dots K]\\
    &= \left(\binom{(M-1)F}{s_1^1-v_1^1,\dots,s_d^1-v_d^1} \frac{1}{d^{(M-1)F}}\right)^K\\
    &\leq p(M-1,F,d)^K,
\end{align*}
where we tacitly assumed that \(s_i^k-v_i^k\geq 0\) for all \(i\) and \(k\); otherwise
\begin{equation*}
    \Pr[\sum_{m=1}^{M-1} X^{mk} + v^k = s^k\quad \text{for all } k=1\dots K] = 0.
\end{equation*}
\end{proof}

In the lemma we assume that all user have the same amount of features (as is the case in our online experiment). However, we only use this for simplifying the notation; one can also do without this assumption and get a corresponding bound. One simply has to replace \(MF\) by the total number of features and \((M-1)F\) by this number minus \(\norm{v}_1\) in the probability bounds.

Taking \(M=34,615,\ F=1826,\ d=95,880\) --- the numbers from the Twitter experiment if we pretend to have the same number of features for each user and assume that 10\% of the users send a positive update ---, the lemma gives a bound of \(<5\times 10^{-173411}\).

\bibliographystyle{ACM-Reference-Format}
\bibliography{references}

\end{document}